\documentclass[twoside]{article}
\usepackage[accepted]{aistats2019}
\usepackage{jmlr2e}
\usepackage{tikz}
\usetikzlibrary{arrows.meta,positioning}
\usepackage{algorithm, float}
\usepackage{algpseudocode}
\floatname{algoirthm}{algorithm}
\algnewcommand\algorithmicinput{\textbf{Input:}}
\algnewcommand\INPUT{\item[\Alg.\icinput]}
\algnewcommand\algoirhtmicoutput{\textbf{Output:}}
\algnewcommand\OUTPUT{\item[\algorithmicoutput]}

\usepackage{amssymb,amsfonts,upref,epsf,color,bm}

\DeclareMathOperator*{\argmin}{arg\;min}

\DeclareMathOperator*{\diag}{diag}

\newcommand\vect[1]{\mathbf #1}
\newcommand{\va}{\vect{a}}  
\newcommand{\vb}{\vect{b}}

\newcommand{\vs}{\vect{s}}  

\newcommand{\vu}{\vect{u}}  
\newcommand{\vv}{\vect{v}}  

\newcommand{\vx}{\vect{x}}
\newcommand{\vy}{\vect{y}}   
\newcommand{\vz}{\vect{z}}

\newcommand{\mD}{\mathbf{D}}
\newcommand{\mL}{\mathbf{L}}

\newcommand{\mW}{\mathbf{W}}

\newcommand{\signalsize}{N}

\newcommand{\graphsigs}{\mathbb{R}^{\mathcal{V}}}

\newcommand{\edges}{\mathcal{E}}
\newcommand{\cluster}{\mathcal{C}}
\newcommand{\nodes}{\mathcal{V}}
\newcommand{\graph}{\mathcal{G}}
\newcommand{\samplingset}{\mathcal{M}}

\newcommand{\expect}{{\rm E}}
\newcommand{\flow}{h}
\newcommand{\partition}{\mathcal{F}}
\newcommand{\boundary}{\partial \partition}

\newcommand{\compbound}{\overline{\partial \partition}}

\newcommand{\prob}{{\rm P}}

\newcommand{\csignal}{a_{\cluster}}

\newcommand\defeq{:=}

\newcommand{\edge}[2]{\{#1,#2\}}

\begin{document}

%

%

\twocolumn[

\aistatstitle{Analysis of Network Lasso for Semi-Supervised Regression}

\aistatsauthor{A. Jung and N. Vesselinova}

\aistatsaddress{Department of Computer Science, Aalto University, Finland } ]

\begin{abstract}
        We apply network Lasso to semi-supervised regression problems involving network-structured data. 
        This approach lends quite naturally to highly scalable learning algorithms in the form of message passing 
        over an empirical graph which represents the network structure of the data. By using a simple non-parametric regression model, 
        which is motivated by a clustering hypothesis, we provide an analysis of the estimation error incurred by network Lasso. 
        This analysis reveals conditions on the the network structure and the available training data which guarantee network Lasso 
        to be accurate. Remarkably, the accuracy of network Lasso is related to the existence of sufficiently large network flows over 
        the empirical graph. Thus, our analysis reveals a connection between network Lasso and maximum flow problems. 
\end{abstract}

\section{INTRODUCTION}
\label{sec_intro}

The datasets arising in many applications, ranging from image processing to cyber security carry an intrinsic network structure. 
In particular, those datasets can be represented conveniently using an empirical graph \cite{SemiSupervisedBook}. 
The nodes of this empirical graph represent individual data points, which are connected by edges according to some domain-specific notion of similarity. 

On top of the network structure, datasets carry additional information in the form of labels for the individual data points. Since the acquisition 
of label information is often expensive (requiring manual labour), we typically have access to labels of few data points only. Moreover, the 
available label information will often be noisy due to measurement (labelling) errors.

The available incomplete label information might still suffice to allow for accurate machine learning by exploiting the 
tendency of labels to conform to the underlying network structure. Indeed, many successful learning methods rely on a clustering 
hypothesis which requires well-connected data points to have similar labels \cite{BishopBook,SemiSupervisedBook}. 

Various generalisations of the least absolute shrinkage and selection operator (Lasso) from sparse vectors to network-structured data 
have been proposed recently by \cite{FusedLasso,SharpnackJMLR2012}. In particular, the ``network Lasso'' (nLasso) \cite{NetworkLasso} provides 
an optimization framework for a wide range of learning problems (regression and classification) involving network-structured datasets. 
While efficient implementations of nLasso for particular learning problems have been proposed (see \cite{pmlr-v54-yamada17a}), 
only little is known about the statistical performance of nLasso methods for general learning problems involving partially labelled 
network-structure data.

{\bf Contribution.} In this paper, we apply a generalization of the concept of a compatibility condition, which has been championed by 
\cite{BuhlGeerBook,vdGeer2007} for characterizing the performance of Lasso methods, to learning problems involving network structured data. 
Various forms of such ``network compatibility conditions'' have been studied recently  by \cite{WhenIsNLASSO,Ortelli2018,pmlr-v49-huetter16}. 
Here, we use a particular form of a network compatibility condition to characterize the performance of nLasso for semi-supervised regression 
problems using squared error loss. The nLasso provides an efficient method for non-parametric regression by leveraging the underlying 
network structure \cite{Kovac20120}. Our results give a precise characterization of the statistical performance of such methods and their 
dependence on the network topology. The closest to our work is \cite{pmlr-v49-huetter16}, which studies the statistical properties of nLasso 
applied to denoising a fully observed graph signal. In contrast, our analysis allows for nLasso having access only to the signal values of a 
small subset (the training set) of nodes, which is relevant for semi-supervised learning problems  (see \cite{SemiSupervisedBook}).  

{\bf Outline.} This paper is organized as follows: in Section \ref{sec_setup}, we formalize the problem of semi-supervised learning 
for network-structured data using a probabilistic model for the observations, which is based on exponential families. 
Based on this generic probabilistic model, we then show in Section \ref{sec_network_lasso} how to apply network Lasso to learn a predictor 
for all data points based on knowledge of noisy labels for few data points. Our main result is discussed in Section \ref{sec_conditions}, 
where we present a bound on the estimation error of nLasso. This bound depends on the network compatibility condition 
which, in turn, relates to the connectivity of sampled nodes.

{\bf Notation.} We use boldface upper and lower case letters to denote matrices and vectors, respectively. 
Given a matrix $\mathbf{W}$ we define its supremum norm as  $\| \mathbf{W} \|_{\infty} \defeq \max_{i,j} |W_{i,j}|$. 
The nullspace (or kernel) of a matrix $\mathbf{L}$ is denoted ${\rm ker} \{ \mathbf{L} \} \defeq \{ \vx : \mL \vx = \mathbf{0} \}$. 
The pseudo-inverse of a diagonal matrix $\mathbf{A}$ is denoted $\mathbf{A}^{\dagger}$ and obtained by inverting the 
non-zero diagonal entries of $\mathbf{A}$ and leaving the zero entries. The pseudo-inverse of an arbitrary matrix $\mathbf{D}$ is 
obtained via its singular value decomposition $\mathbf{D} = \mathbf{U} {\bf \Lambda} \mathbf{V}^{T}$ as $\mathbf{D} = \mathbf{U} {\bf \Lambda}^{\dagger} \mathbf{V}^{T}$. 
Given a finite set $\mathcal{V}$, we denote the complement of a subset $\mathcal{M} \subseteq \nodes$ as $\overline{\mathcal{M}}$.

\section{PROBLEM FORMULATION}
\label{sec_setup}

We consider network-structured datasets, which are represented by an empirical graph $\graph\!=\!(\nodes,\edges, \mathbf{W})$. 
The nodes $\nodes = \{1, \ldots, {\signalsize}\}$ of the empirical graph represent individual data points.
The undirected edges $\edges$ encode domain-specific notions of similarity between data points. The non-negative 
entries $W_{i,j}$ of the weight matrix $\mW \in \mathbb{R}_{+}^{\signalsize \times \signalsize}$ quantify the level of 
similarity between connected nodes. The weight $W_{i,j}$ is non-zero only if nodes $i,j \in \nodes$ are 
connected by an edge $\{i,j\} \in \edges$. 


%
%
%

In what follows, without loss of generality, we assume that the empirical graph is simple (without self loops) and connected. 
Therefore, since there are no self loops, the weight matrix is such that $W_{i,i} = 0$ for every node $i \in \nodes$.

\subsection{Laplacian and Incidence Matrix}

The structure of an empirical graph $\graph$ can be characterized using the graph Laplacian matrix 
\begin{equation}
\label{equ_def_graph_Laplacian_matrix}
\mathbf{L}= {\bf \Lambda} - \mW, 
\end{equation}
with the weight matrix $\mathbf{W}$ and the diagonal ``degree matrix'' 
\begin{equation}
\nonumber
{\bf \Lambda} = {\rm diag} \{ d_{1},\ldots,d_{\signalsize} \} \in \mathbb{R}^{\signalsize \times \signalsize}. 
\end{equation}
The diagonal elements of ${\bf \Lambda}$ are the weighted node degrees $d_{i} \defeq \sum_{\{j,i\} \in \edges} W_{i,j}$.

The non-negative eigenvalues $\lambda_{1}\!\leq\!\ldots\!\leq\!\lambda_{\signalsize}$ of the Laplacian matrix $\mathbf{L}$ 
provide insight into the connectivity structure of the graph $\graph$: for a connected graph $\graph$, the smallest 
eigenvalue is $\lambda_{1}=0$. The nullspace of $\mathbf{L}$ is a one-dimensional subspace spanned by 
the constant graph signal with value $x[i]=1$ for every node $i \in \nodes$. The spectral gap $\rho (\graph) \defeq \lambda_{2}$ 
quantifies the connectivity of the graph $\graph$. If $\rho(\graph)$ is close to zero, the graph $\graph$ can be cut into 
two disconnected subgraphs without removing too many edges \cite{SpielSGT2012}.

Another important matrix assigned to an empirical graph is the incidence matrix. To this end, we (arbitrarily) orient  
the empirical graph $\graph=(\edges,\nodes,\mW)$ by specifying for each edge $e = \{i,j\}$ one node as the head 
$e^{+}$ and the other node as the tail $e^{-}$. We define the incidence matrix 
$\mD \in \mathbb{R}^{\edges \times \nodes}$ element-wise as
\begin{equation}
\label{equ_def_incidence_matrix_elementwise}
 D_{e,i}= \begin{cases} \sqrt{W_{e}} & \mbox{ if }  i\!=\!e^{+} \\ 
 - \sqrt{W_{e}} & \mbox{ if } i\!=\!e^{-}  \\ 
0 & \mbox{ else.} \end{cases} 
\end{equation}
We highlight that the exact choice of orientation for the undirected edges in the empirical graph $\graph$ has 
no effect on our results. The use of an orientation only serves a notational convenience provided by 
the incidence matrix $\mD$. 

The incidence matrix $\mD$ is closely related to the graph Laplacian $\mL$. Indeed, both matrices have 
the same nullspace ${\rm ker } \{ \mD \} = {\rm ker } \{ \mL \}$. Moreover, the spectrum of $\mathbf{D} \mathbf{D}^{T}$ 
coincides with the spectrum of $\mL^{(\graph)}$. The columns $\vs_{j}$ of the pseudo-inverse $\mD^{\dagger}\!=\!(\vs_{1},\ldots,\vs_{|\edges|})$ 
of $\mD$ satisfy 
\begin{equation}
\label{equ_bound_cols_pseudo_D}
\| \vs_{j} \| \leq \sqrt{2 \| \mathbf{W} \|_{\infty} }/\rho(\graph).
\end{equation}
This bound can verified using the identity $\mD^{\dagger}\!=\!(\mD \mD^{T})^{\dagger} \mD^{T}$ and well-known 
vector norm inequalities (see, e.g., \cite{Horn85}).

\subsection{Linear Regression}
\label{sec_graph_signals}
In addition to the network structure, which is encoded by the empirical graph $\graph$, datasets typically 
convey additional information. This additional information comes in the form of labels $y_{i}$ associated with individual data points $i \in \nodes$. 

We model the labels $y_{i}$ of data points $i \in \nodes$ as random variables whose probability distribution is parametrized 
by a graph signal $\bar{\vx}: \nodes \rightarrow \mathbb{R}$. In particular, we use the linear model 
\begin{equation}
\label{equ_obs_model}
y_{i} = \bar{x}_{i} + \varepsilon_{i}, 
\end{equation} 
with some unknown underlying graph signal $\bar{\vx}$. 
The noise terms $\varepsilon_{i}$ in \eqref{equ_obs_model} are modelled as i.i.d.\ 
Gaussian random variables with zero-mean and variance $\sigma^{2}$, cover any modelling or measurement (labeling) errors. 
We will use the following tail bound  
\begin{align} 
\label{equ_concentration_lin_reg_sum_opt}
\hspace*{-2mm}\prob \{ | y\!-\! \expect\{y\}|\!\geq\!\eta \}& \!\leq\! \nonumber  \\
& \hspace*{-30mm} 2 \exp \hspace*{-1mm}\Bigg(\hspace*{-2mm}-\signalsize^2\eta^{2}/\bigg(2\sigma^{2}\sum_{i=1}^{\signalsize} w^{2}_{i} \bigg) \Bigg), 
\end{align}
for the weighted sum $y= (1/\signalsize) \sum_{i=1}^{\signalsize} y_{i} w_{i}$ with 
arbitrary but fixed weights $w_{i} \in \mathbb{R}$.

The graph signal $\bar{\vx}$ in \eqref{equ_obs_model} assigns a real number $\bar{x}_{i} \!\in\! \mathbb{R}$ to each node $i\!\in\!\nodes$. 
We can think of a graph signal also as a vector whose entries are indexed by the nodes $i\in \nodes$. The space of all graph signals constitutes 
an Euclidean space $\graphsigs$. It will be convenient to define, for a given subset $\samplingset \subseteq \nodes$, the norm 
\begin{equation}
\| \vx \|_{\samplingset} \defeq \sqrt{(1/|\samplingset) \sum_{i \in \samplingset} x^{2}_{i}}. 
\end{equation} 

Since acquiring labels is costly, we consider having access to the (noisy) labels $y_{i}$ (see \eqref{equ_obs_model})  
only for the nodes in a (small) training set 
\begin{equation}
\label{equ_def_training_set}
\samplingset\!=\!\{i_{1},\ldots,i_{M} \} \mbox{ with }M\!\ll\!\signalsize.
\end{equation}

\subsection{Clustering Hypothesis}

Our approach to learning the graph signal $\bar{\vx}$ in \eqref{equ_obs_model} from the labels $y_{i}$ of the nodes in the training set $\samplingset$, 
is based on the assumption that the graph signal $\bar{\vx}$ is clustered in the sense of being constant over well-connected subsets (clusters) of nodes. 
This clustering hypothesis conforms to the finding that the labels of data points arising many application domains, such as signal or image processing as well as 
social networks, are similar if the data points are well-connected in the empirical graph (see \cite{SemiSupervisedBook}). 

We measure the amount by which a graph signal $\vx$ conforms with the cluster structure of the empirical graph $\graph$ 
using the (weighted) total variation (TV)
\begin{equation}
\label{equ_def_TV_norm}
\| \vx \|_{\rm TV} \defeq \sum_{\{i,j\}\in \edges} \sqrt{W_{ij}} \left| x_{j} - x_{i} \right|.
\end{equation} 

Indeed, a graph signal $x[\cdot]$ has a small TV only if the signal values $x[i]$ are approximately constant 
over well connected subsets (clusters) of nodes. Such a ``clustering hypothesis'' (or variations thereof) motivates many 
methods for (semi-) supervised learning \cite{SemiSupervisedBook}.

If we orient the empirical graph, we can represent the TV using the incidence matrix 
$\mD$ (see \eqref{equ_def_incidence_matrix_elementwise} and \eqref{equ_def_TV_norm}) 
as 
\begin{equation}
\label{equ_ident_TV_1_norm_D}
\| \vx \|_{\rm TV} = \| \mD \vx \|_{1}.  
\end{equation} 
It will be convenient to define 
a shorthand for the TV over a subset $\mathcal{S} \subseteq \edges$ of edges as 
\begin{equation} 
\label{equ_def_TV_norm_subset}
\| \vx \|_{\mathcal{S}} \defeq \sum_{\{i,j\}\in \mathcal{S}} \sqrt{W_{ij}} \left| x_{j} - x_{i} \right|.
\end{equation}
One of our main contributions (see Section \ref{sec_conditions}) is a precise analysis of the ability of nLasso to learn 
clustered graph signals. In particular, our analysis is based on the following simple model 
for clustered (piece-wise constant) graph signals (see \cite{ChenYMK16}):
\begin{equation}
\label{equ_def_clustered_signal_model}
x_{i} \!=\! \sum_{\cluster \in \partition} \csignal \mathcal{I}_{\cluster}[i].
\end{equation} 
Here, $\csignal \in \mathbb{R}$ is the signal value of cluster $\cluster$ and we used the indicator signal 
\begin{equation} 
\nonumber
\mathcal{I}_{\cluster}[i] = \begin{cases} 1 & \mbox{ for } i \in \cluster \\
0 & \mbox{ otherwise.} \end{cases}
\end{equation}

The model \eqref{equ_def_clustered_signal_model} involves a partitioning $\partition = \{\cluster_{1}, \ldots, \cluster_{|\partition|}\}$ of 
the nodes $\nodes$ into disjoint subsets $\cluster_{l}$. We assume that the subgraph induced by any cluster $\mathcal{C}_{l}$ is connected. 

While our analysis allows for an arbitrary partitioning $\partition = \{ \cluster_{1},\ldots,\cluster_{|\mathcal{\partition}|} \}$ used 
to define the model \eqref{equ_def_clustered_signal_model}, our results are most useful if the partition conforms with the 
``intrinsic (cluster) structure'' of the empirical graph $\graph$. In particular, we focus on partitions $\partition$ such that the cluster boundaries
\begin{equation*}
\partial \partition \defeq \{ \edge{i}{j} \in \edges: i \in \cluster, j \in \cluster' (\neq \cluster) \}
\end{equation*} 
satisfy $\sum\limits_{\edge{i}{j} \in \partial \partition }\hspace*{-2mm} \sqrt{W_{i,j}} \!\ll\! \sum\limits_{\edge{i}{j} \in \overline{\partial \partition}} \hspace*{-2mm}\sqrt{W_{i,j}}$. 

It will be useful to define the spectral gap of a partitioning $\partition = \{\cluster_{1}, \ldots, \cluster_{|\partition|}\}$ as
\begin{equation}
\label{equ_partition_spectral_gap}
\rho(\partition)  \defeq \min_{\cluster_{l} \in \partition} \rho(\cluster_{l}). 
\end{equation}
Here, $\rho(\cluster_{l})$ denotes the spectral gap of the subgraph induced 
by the cluster $\cluster_{l}$.

\section{THE NETWORK LASSO}
\label{sec_network_lasso} 

It is sensible to learn a graph signal $\hat{\vx} \in \mathbb{R}^{\nodes}$ based on (few) 
labels $\{ y_{i} \}_{i \in \samplingset}$ is to maximize the probability (``evidence'') $\prob \{ \{y_{i} \}_{i \in \samplingset}; \vx \}$ 
of observing them under the probabilistic model \eqref{equ_obs_model} for the labels. 
This is equivalent to minimizing the empirical error:  
\begin{equation} 
\label{equ_def_emp_risk}
\widehat{E}(\vx)  \defeq  (1/M) \sum_{i \in \samplingset} ( y_{i} - x_{i} )^2.
\end{equation}
The criterion \eqref{equ_def_emp_risk} by itself is not sufficient for guiding the learning of a graph signal based 
on few labels $\{ y_{i} \}_{i \in \samplingset}$, since it ignores 
the signal values $x_{i}$ for $i \in \overline{\samplingset}$.

In order to learn an entire graph signal $\hat{\vx}$ from the incomplete information provided by the initial labels $\{ y_{i} \}_{i \in \samplingset}$, 
we need to impose some structure on the graph signal $\hat{\vx}$. This additional structure is provided by the empirical graph $\graph$. 
In particular, we assume that any reasonable graph signal $\hat{\vx}$ needs to conform with the 
\emph{cluster structure} of $\graph$ (see \cite{NewmannBook}).

We are led quite naturally to learning a graph signal $\hat{\vx}$ by balancing a small empirical error (risk) $\widehat{E}(\hat{\vx})$ 
(see \eqref{equ_def_emp_risk}) with a small TV $\| \hat{\vx} \|_{\rm TV}$ (see \eqref{equ_def_TV_norm}). 
Thus, we arrive at the following \emph{regularized empirical risk minimization} 
\begin{align} \label{optProb}
\hat{\vx} & \in \argmin_{\vx \in \graphsigs} \widehat{E}(\vx)  + \lambda \| \vx \|_{\rm TV}.
\end{align}

The parameter $\lambda$ in \eqref{optProb} allows to trade-off a small TV $\| \hat{\vx} \|_{\rm TV}$ against 
a small empirical error. 
Choosing a small value of $\lambda$ will result in a graph signal $\hat{\vx}$ with small empirical 
error $\widehat{E}(\hat{\vx})$ (see \eqref{equ_def_emp_risk}), while choosing a large value of 
$\lambda$ favours $\hat{\vx}$ with small TV $\| \hat{\vx} \|_{\rm TV}$ (being more clustered).

The learning problem \eqref{optProb} is a particular instance of the nLasso 
introduced in \cite{NetworkLasso} which allows for efficient implementations using 
modern convex optimization methods \cite{ProximalMethods,DistrOptStatistLearningADMM}. 
In particular, we obtain Algorithm \ref{algo_nLasso} by applying the primal-dual method 
proposed by \cite{PrecPockChambolle2011} to 
\begin{align}  
\hat{\vx}& \in \argmin_{\vx \in \graphsigs} \widehat{E}(\vx)  + \lambda \| \mathbf{D} \vx \|_{1} \nonumber \\ 
& = \argmin_{\vx \in \graphsigs} \max_{\|\vu \|_{\infty} \leq 1}  \widehat{E}(\vx)  + \lambda \mathbf{u}^{T} \mathbf{D} \vx  \nonumber
\end{align} 
which, due to \eqref{equ_ident_TV_1_norm_D}, is equivalent to \eqref{optProb}. 

\begin{algorithm}
		\smallskip
		\textbf{Input:} $\mD \in \mathbb{R}^{\edges \times \nodes}$, $\samplingset$,  $\{y_{i}\}_{i \in \samplingset}$, $\lambda$\\
	
		{\bf Init}: $k\!\defeq\!0,\bar{\vx}\!=\!\hat{\vx}^{(-1)}\!=\!\hat{\vx}^{(0)}\!=\!\hat{\vy}^{(0)}\!\defeq\! \mathbf{0}$, \\ 
		$\nu\!\defeq\!1/(\lambda|\samplingset|), \gamma_{i}\!\defeq\! \sum_{j \in \nodes} \sqrt{W_{i,j}}$ \\ 
		$ {\bm \Gamma}\!\defeq\!\diag \{1/\gamma_{i} \}\!\in\!\mathbb{R}^{\nodes \times \nodes}$ \\ 
	        ${\bm \Lambda}\!\defeq\!\diag \{ 1/(2 \sqrt{W_{i,j}}) \}\!\in\!\mathbb{R}^{\edges \times \edges}$, \\

		\textbf{repeat:}\\
		1: $\vx  \defeq 2 \hat{\vx}^{(k)} - \hat{\vx}^{(k-1)}$\\[1mm]
		2: $\hat{\vz}  \defeq \hat{\vy}^{(k)} + {\bf \Lambda}  \mD  \vx$ \\[1mm]
		3: $\hat{y}_{e}^{(k\!+\!1)}  \defeq \hat{z}_{e} / \max\{1, |\hat{z}_{e}| \}$  for all  $e\!\in\!\edges$ \\[1mm]
                 4: $\hat{\vx}^{(k+1)}  \defeq \hat{\vx}^{(k)} - {\bm \Gamma} \mD^{T} \hat{\vy}^{(k\!+\!1)}$ \\[1mm]
                 5: $\hat{x}_{i}^{(k+1)} \defeq \frac{2\nu y_{i}\!+\!\gamma_{i} \hat{x}_{i}^{(k\!+\!1)}}{2\nu+\gamma_{i}}$  for all  $i \in \samplingset$\\[1mm]
                 6: $k \defeq k+1$ \\[1mm]
                 7: $\bar{\vx}^{(k)} \defeq (1\!-\!1/k) \bar{\vx}^{(k\!-\!1)}\!+\!(1/k) \hat{\vx}^{(k)} $\\[1mm]
                 \textbf{until} stopping criterion is satisfied\\[0mm]
            \textbf{Output}:  labels $\hat{x}_{i} \defeq \bar{x}^{(k)}_{i}$ \mbox{ for all }$i \in \nodes$			
	\caption{}
	\label{algo_nLasso}
\end{algorithm}

\section{STATISTICAL PROPERTIES OF NETWORK LASSO}
\label{sec_conditions}

The accuracy of nLasso methods depends on how close the solutions $\hat{\mathbf{x}}$ of \eqref{optProb} are to the true 
underlying clustered graph signal $\bar{\mathbf{x}} \in \graphsigs$ (see \eqref{equ_obs_model} and \eqref{equ_def_clustered_signal_model}). 

In what follows, we derive a condition on the cluster structure $\partition$ and training set $\samplingset$, 
which guarantee any solution $\hat{\mathbf{x}}$ of \eqref{optProb} to be close to the underlying graph signal 
$\bar{\mathbf{x}}$. This condition, which we refer to as network compatibility condition (NCC) extends the concept of 
compatibility conditions used for analyzing Lasso methods for learning sparse vectors \cite{GeerBuhlConditions}, 
to network-structured data. We then show that this network compatibility condition is related to the existence 
of a sufficiently large network flow. The existence of such network flows indirectly characterizes the connectivity 
of sampled nodes $\samplingset$ in different clusters via the cluster boundaries $\boundary$.

\subsection{Flows over the Empirical Graph}

The main conceptual contribution of this paper is the insight that the accuracy of nLasso methods, aiming at solving \eqref{optProb}, depends 
on the topology of the underlying empirical graph via the existence of certain \emph{flows with demands} \cite{KleinbergTardos2006}. 

A flow over the empirical graph $\graph$ is a mapping $\flow: \nodes \times \nodes \rightarrow \mathbb{R}$ 
which assigns each directed edge $(i,j)$ the value $\flow(i,j)$, which can be interpreted as the amount of some 
quantity flowing through the edge $(i,j)$ (see \cite{KleinbergTardos2006}).

A flow with demands has to satisfy the conservation law 
\begin{equation} 
\label{equ_conservation_flow}
\hspace*{-6mm}\sum_{j \in \mathcal{N}(i)}\hspace*{-2mm} \flow(i,j)   = f[i] \mbox{, for any }  i \!\in\! \nodes
\end{equation}
with a prescribed demand $f[i]$ for each node $i \in \nodes$. 
Moreover, we require flows to satisfy the capacity constraints  
\begin{equation} 
\label{equ_cap_constraint}
|\flow (i,j)|\!\leq\!\sqrt{W_{i,j}} \mbox{ for any  } (i,j) \!\in\! \overline{\boundary}. 
\end{equation} 
Note that the capacity constraint \eqref{equ_cap_constraint} applies only to intra-cluster edges 
and does not involve the boundary edges $\boundary$. The flow values $\flow(i,j)$ at the boundary edges $(i,j) \in \boundary$ 
take a special role in the following definition of the notion of resolving training sets. 
\begin{definition}
	\label{def_sampling_set_resolves}
	Consider an empirical graph $\graph = (\nodes, \edges,\mathbf{W})$ and a partition $\partition=\{\cluster_{1},\ldots,\cluster_{|\partition|}\}$. 
	A (training) set  $\samplingset=\{i_{1},\ldots,i_{M}\} \subseteq \nodes$ resolves $\partition$ with constants $K,L>0$ if, for any $b_{i,j} \in \{-1,1\}^{\boundary}$, 
	there is a flow $\flow[\cdot]$ on $\graph$ (cf.\ \eqref{equ_conservation_flow}, \eqref{equ_cap_constraint}) with $\flow(i,j)\!=\!b_{i,j} L\sqrt{W_{i,j}}$ 
	for $\{i,j\} \in \boundary$ and demands (cf.\ \eqref{equ_conservation_flow}) $|f[i]| \!\leq\! K/M$ for $i\!\in\! \samplingset$ and $f[i]\!=\!0$ for $i\!\in\!\overline{\samplingset}$. 
\end{definition}

This definition requires nodes of a resolving training set to be sufficiently well connected with each boundary edge $\{i,j\}\!\in\!\boundary$. 
In particular, we could think of injecting (absorbing) certain amounts of flow into (from) the empirical graph at the sampled nodes. 
At each sampled node $i\in \samplingset$, we can inject (absorb) a flow of value at most $K/M$. 
The injected (absorbed) flow has to be routed from the sampled nodes $\samplingset$ via the intra-cluster edges $\overline{\boundary}$ 
to each boundary edge $\{i,j\} \in \boundary$ such that it carries a flow value $L \cdot \sqrt{W_{i,j}}$.

Note that Definition \ref{def_sampling_set_resolves} is quantitive as it involves the numerical constants $K$ and $L$. Our main result stated below is an 
upper bound on the estimation error of nLasso methods, which depends on the value of these constants. It will turn out that resolving sampling 
sets with a small values of $K$ and large values of $L$ are beneficial for the ability of nLasso to recover the entire graph signal from noisy 
samples $\{y_{i} \}_{i \in \samplingset}$ observed on the training set $\samplingset$.

\subsection{Linear Regression with nLasso}

For the analysis of the nLasso problem \eqref{optProb}, we will make use of the network compatibility condition (NCC) defined as follows. 

\begin{definition} 
	\label{def_NNSP}
	Consider an empirical graph $\graph = (\nodes, \edges, \mathbf{W})$ with a particular partition $\partition$ of its nodes $\nodes$. 
	A sampling set $\samplingset \subseteq \nodes$ is said to satisfy NCC with constants $K,L>1$, if 
	\vspace*{-1mm}
	\begin{equation}
	\label{equ_ineq_multcompcondition_condition}
	L  \| \vz \|_{\partial \partition}  \leq K  \| \vz \|_{\samplingset}+  \| \mathbf{z} \|_{\compbound}
	\vspace*{-1mm}
	\end{equation} 
	for any graph signal $\mathbf{z} \!\in\!\mathbb{R}^{\nodes}$. 
\end{definition} 

The NCC guarantees nLasso \eqref{optProb} to accurately recover graph signals of the 
form \eqref{equ_def_clustered_signal_model}. Note that the NCC involves the partition $\partition$ underlying 
the signal model \eqref{equ_def_clustered_signal_model}. However, the partition is not required for the implementation of nLasso \eqref{optProb}. 

It turns out that the resolving sets (see Definition \ref{def_sampling_set_resolves}) satisfy the NCC.

\begin{lemma}
\label{lem_resolving_implies_NCC}
	Consider an empirical graph $\graph$ whose nodes are partitioned as $\partition =\{\cluster_{1},\ldots,\cluster_{|\partition|}\}$. 
	If a set $\samplingset$ resolves $\partition$, it 
	satisfies NCC with the same parameters $K,L$. 
\end{lemma} 
\begin{proof}
	The statement follows easily from \cite[Lemma 6]{WhenIsNLASSO} and the 
	Cauchy-Schwarz inequality, which implies $\sum\limits_{i \in \samplingset }|z_{i}| \leq \sqrt{M \sum\limits_{i \in \samplingset} z_{i}^{2}}$. 
\end{proof}

Our main result is that the NCC, with suitable constants $L$ and $K$, implies that solutions of the 
nLasso problem \eqref{optProb} are close to the true underlying clustered graph signal $\bar{\mathbf{x}}$ (cf.\ \eqref{equ_def_clustered_signal_model}).

\begin{theorem} 
	\label{lem_NSP1}
	Consider an empirical graph $\graph$, whose nodes have labels $y_{i}$ distributed according to \eqref{equ_obs_model} with 
	underlying clustered graph signal $\bar{\vx}$ \eqref{equ_def_clustered_signal_model}. We estimate the 
	underlying graph signal $\bar{\mathbf{x}}$ using $\hat{\mathbf{x}}$ obtained from solving the nLasso problem \eqref{optProb}. 
	If the training set $\samplingset$ satisfies the NCC with parameters $L > 4$, and $K \in (1,L-2)$ and 
	condition number $\kappa \defeq \frac{K\!+\!3}{L\!-\!3}$ (see Definition \ref{def_NNSP}), 
	\begin{align}
	\label{lower_boung_prob_main_results}
	\prob \{  \| \hat{\mathbf{x}}-\bar{\mathbf{x}} \|_{\rm TV}  \!\geq\! \eta\} & \leq  2 |\partition| \exp \bigg( - \frac{|\cluster_{l}|\eta^{2}}{6300\kappa^{2} \sigma^{2}} \bigg) \nonumber \\
	& \hspace*{-20mm}\!+\!2M\exp \bigg(\!-\! \frac{M^2 \rho^{2}_{\partition}  \eta^{2}}{900 \kappa^2 \sigma^2  \| \mathbf{D} \|^{2}_{\infty} }\bigg) .
	\end{align}
\end{theorem}

The bound \eqref{lower_boung_prob_main_results} indicates that, for a prescribed accuracy level $\eta$, 
the training set size $M$ has to scale according to $\kappa \sigma / \rho_{\partition}$. 
Thus, the sample size required by Algorithm \ref{algo_nLasso} scales linearly with the condition number $\kappa =  \frac{K\!+\!3}{L\!-\!3}$ (see Definition \ref{def_NNSP}) and 
inversely with the spectral gap $\rho_{\partition}$ of the partitioning $\partition$. Thus, nLasso methods \eqref{optProb} (such as Algorithm \ref{algo_nLasso}) require less 
training data if the condition number $\kappa$ is small and the spectral gap $\rho_{\partition}$ is large. This is reasonable, since according to 
Lemma \ref{lem_resolving_implies_NCC}, a small condition number (NCC parameter $L$ is large compared to $K$) requires the edges within 
clusters to have higher weights on overage than the weights of the boundary edges. Moreover, it is reasonable that nLasso tends to be more accurate for a larger spectral gap 
$\rho_{\partition}$, which requires the nodes within each cluster $\cluster_{l}$ to be well connected. Indeed, an empirical graph $\graph$ consisting of well-connected clusters $\cluster_{l}$ 
favours clustered graph signals, such as the true underlying graph signal $\bar{\vx}$ in \eqref{equ_obs_model}, to be solutions of the nLasso \eqref{optProb}.


\subsection{Proof of Theorem \ref{lem_NSP1}} 

	By following the reasoning pattern in \cite{SelfConcLogReg} and \cite{BuhlGeerBook}, we organize the proof in two parts. 
	The first part is to verify that, with high probability, the estimation error $\tilde{\vx}\!\defeq\!\bar{\vx}\!-\!\hat{\vx}$ incurred 
	by nLasso \eqref{optProb} is approximately clustered according to \eqref{equ_def_clustered_signal_model}. 
	The second part is to upper bound the nLasso error $\tilde{\vx}$ using the NCC \eqref{equ_ineq_multcompcondition_condition}.
	
	First, any solution $\hat{\vx}$ of \eqref{optProb} satisfies 
	\vspace*{0mm}
	\begin{align}
	\label{equ_basic_def_lasso_proof}
	(1/M)\sum_{i \in \samplingset} & \hspace*{-1mm}\big[\!- y_{i}(\hat{x}_{i}\!-\!\bar{x}_{i})\!+\!(1/2)(\hat{x}_{i}^{2}\!-\!\bar{x}_{i}^{2})\big]  \nonumber \\[-1mm]
	& \leq\! (\lambda/2) ( \| \bar{\vx} \|_{\rm TV}-\| \hat{\vx} \|_{\rm TV}).  \\[-4mm]
	\nonumber
	\end{align} 
	From \eqref{equ_basic_def_lasso_proof}, we get (see \eqref{equ_obs_model})
	\begin{equation}
	\label{equ_basic_def_lasso_proof3}
	(1/M) \sum_{i \in \samplingset} \hspace*{-1mm} \varepsilon_{i} \tilde{x}_{i} \!+\!\lambda \| \hat{\vx} \|_{\rm TV} \leq (\lambda/2) \| \bar{\vx} \|_{\rm TV}.
	\end{equation} 
	
	Assume the noise $\varepsilon_{i}$ is small such that 
	\begin{align}
	\label{equ_small_noise_condition}
	\big| \frac{1}{M}\sum_{i \in \samplingset} \hspace*{-1mm} \varepsilon_{i} \tilde{x}_{i} \big| &\leq  \lambda \kappa  \| \tilde{\vx} \|_{\samplingset}\!+\!(\lambda/2) \| \tilde{\vx} \|_{\rm TV}
	\end{align} 
	holds for every graph signal $\tilde{\vx}\in \mathbb{R}^{\nodes}$.

	Combining \eqref{equ_small_noise_condition} with \eqref{equ_basic_def_lasso_proof3}, 
	\begin{align}
	\| \hat{\vx} \|_{\rm TV}&\!\leq\!\frac{1}{2} (\| \tilde{\vx} \|_{\rm TV}\!+\!\| \bar{\vx} \|_{\rm TV})\!+\!\kappa \| \tilde{\vx} \|_{\samplingset}  \nonumber  
	\end{align} 
	and, in turn, via the decomposition property $ \| \vx \|_{\rm TV}\!=\!\| \vx  \|_{\boundary}\!+\!\| \vx \|_{\overline{\boundary}}$ (see \eqref{equ_def_TV_norm_subset}), 
	\begin{align}
	\label{equ_proof_nLasso_is_sparse_reg}
	\| \hat{\vx} \|_{\overline{\boundary}} & \!\leq\! \nonumber \\
	& \hspace*{-10mm} (1/2) (\| \tilde{\vx} \|_{\rm TV} \!+\! \| \bar{\vx} \|_{\rm TV})\!-\!\| \hat{\vx} \|_{\boundary}\!+\! \kappa \| \tilde{x} \|_{\samplingset}  \nonumber \\ 
	& \hspace*{-10mm} \stackrel{(a)}{\leq} (1/2) (\| \tilde{\vx} \|_{\rm TV} \!+\! \| \bar{\vx} \|_{\boundary})\!-\!\| \hat{\vx} \|_{\boundary}\!+\! \kappa \| \tilde{x} \|_{\samplingset} \nonumber \\
	& \hspace*{-10mm} \stackrel{(b)}{\leq} (1/2) \| \tilde{\vx} \|_{\rm TV}\!+\! \| \bar{\vx}\!-\!\hat{\vx}\|_{\boundary}\!+\! \kappa \| \tilde{\vx} \|_{\samplingset}, 
	\end{align} 
	where step $(a)$ is valid since we assume the true underlying graph signal $\bar{\vx}$ to be clustered according to \eqref{equ_def_clustered_signal_model}. 
	In step $(b)$ we used the (reverse) triangle inequality for the semi-norm $\| \cdot \|_{\boundary}$. 
	
	Inserting $\| \hat{\vx} \|_{\overline{\boundary}}\!=\!\| \tilde{\vx} \|_{\overline{\boundary}}$ into \eqref{equ_proof_nLasso_is_sparse_reg} yields
	\begin{align}
	\label{equ_proof_nLasso_is_sparse_reg1}
	\| \tilde{\vx} \|_{\overline{\boundary}} \!\leq\!  3 \|\tilde{\vx}\|_{\boundary}\!+\! 2\kappa\| \tilde{\vx} \|_{\samplingset}.
	\end{align}  
	Thus, for sufficiently small observation noise $\varepsilon_{i}$ (such that \eqref{equ_small_noise_condition} is valid), the nLasso error 
	$\tilde{\vx}\!=\!\hat{\vx}\!-\!\bar{\vx}$ is approximately clustered according to \eqref{equ_def_clustered_signal_model}. 
	
	The next step is to control the nLasso error $\tilde{\vx} = \hat{\vx}- \bar{\vx}$ (see \eqref{optProb}). 
	According to \eqref{equ_basic_def_lasso_proof},
	\begin{align}
	\label{equ_basic_def_lasso_proof11}
	\hspace*{-4mm}(1/M)\!\sum_{i \in \samplingset} \hspace*{-1mm}\big[\!-\!\varepsilon_{i} \tilde{x}_{i}\!+\!\tilde{x}_{i}^2\big] \!+\!\lambda \| \hat{\vx} \|_{\rm TV}\!\leq\!\lambda \| \bar{\vx} \|_{\rm TV}.
	\end{align} 
	Using the (reverse) triangle inequality for the TV semi-norm $\| \cdot \|_{\boundary}$ (see \eqref{equ_def_TV_norm_subset}), \eqref{equ_basic_def_lasso_proof11} becomes 
	\begin{align}
	\label{equ_basic_def_lasso_proof111}
	\hspace*{-2.5mm}(1/M)\!\sum_{i \in \samplingset} \hspace*{-1mm}\big[\!-\!\varepsilon_{i} \tilde{x}_{i}\!+\!\tilde{x}_{i}^2\big] \!\leq\!\lambda \| \tilde{\vx} \|_{\boundary}.
	\end{align} 
	Inserting \eqref{equ_small_noise_condition} into \eqref{equ_basic_def_lasso_proof111}, 
	\begin{align}
	\label{equ_basic_def_lasso_proof1112}
     \| \tilde{\vx} \|^{2}_{\samplingset} \!\leq\!\lambda \| \tilde{\vx} \|_{\boundary}\!+\!\kappa \lambda \| \tilde{\vx} \|_{\samplingset}.
	\end{align} 
	Combining \eqref{equ_proof_nLasso_is_sparse_reg1} with \eqref{equ_ineq_multcompcondition_condition} yields
	\begin{align}
	\label{equ_bound_compat_condition_proof}
	\| \tilde{\vx} \|_{\boundary}  \leq \frac{K+2\kappa}{L-3} \| \tilde{\vx} \|_{\samplingset} \stackrel{(b)}{\leq} 3 \kappa  \| \tilde{\vx} \|_{\samplingset},
	\end{align}
	where step $(b)$ is due to $L > 3$. Combining \eqref{equ_bound_compat_condition_proof} with \eqref{equ_basic_def_lasso_proof1112}, 
	\begin{equation}
	\| \tilde{\vx} \|^{2}_{\samplingset} \!\leq\! 4\lambda  \kappa \| \tilde{\vx} \|_{\samplingset}, 
	\end{equation} 
	and, in turn, 
	\begin{equation}
	\label{equ_bound_lambda_kappa}
	\| \tilde{\vx} \|_{\samplingset} \!\leq\! 4 \lambda  \kappa.
	\end{equation} 
	Inserting \eqref{equ_bound_lambda_kappa} into \eqref{equ_bound_compat_condition_proof} and \eqref{equ_proof_nLasso_is_sparse_reg1}, yields $\| \tilde{\vx} \|_{\rm TV} \leq 56 \lambda  \kappa^2$.
	
	The proof is completed by bounding the probability of \eqref{equ_small_noise_condition} 
	to hold. By Corollary \ref{cor_bound_noise_term_spectralgap}, \eqref{equ_small_noise_condition} holds if 
	\begin{align}
	\label{equ_max_cluster_l_sum_noise}
	\max_{\cluster_{l} \in \partition} (1/|\cluster_{l}|) \sum_{i \in \cluster_l} {\bf \varepsilon}_{i} \leq \lambda \kappa,
	\end{align} 
	and simultaneously 
	\begin{align}
	\label{equ_condition_max_infty_norm_transformed_noise}
	\max_{\cluster_{l} \in \partition} \big\| \big( \mD_{\cluster_{l}}^{\dagger} \big)^{T} {\bm \varepsilon}_{\cluster_{l}} \big\|_{\infty} \leq M \lambda/2. 
	\end{align} 
	
	We first bound the probability that \eqref{equ_max_cluster_l_sum_noise} fails to hold. For a particular 
	cluster $\cluster_{l}$, \eqref{equ_concentration_lin_reg_sum_opt} yields 
	\begin{equation}
	\label{equ_bound_sum_noise_cluster}
	\hspace*{-1.5mm}\prob \{ (1/|\cluster_{l}|)\!\sum_{i \in \cluster_l} \! \varepsilon_{i} \!\geq\! \lambda \kappa\} \!\leq\! 2 \exp\bigg(\hspace*{-2mm}-\! \frac{|\cluster_{l}|\lambda^2\kappa^2}{2\sigma^{2}} \bigg).
	\end{equation}
	Applying a union bound to \eqref{equ_bound_sum_noise_cluster} yields
	\begin{equation}
	\label{equ_final_prob_bound_1}
	\hspace*{-1.5mm}\prob \{\mbox{``\eqref{equ_max_cluster_l_sum_noise} invalid''} \}\!\leq\!2 |\partition| \exp\bigg(\hspace*{-2mm}-\! \frac{|\cluster_{l}|\lambda^2\kappa^2}{2\sigma^{2}} \bigg).
	\end{equation}
	
	For controlling the probability of \eqref{equ_condition_max_infty_norm_transformed_noise} failing to hold, we note that the entries of 
	$\big( \mD_{\cluster_{l}}^{\dagger} \big)^{T} {\bm \varepsilon}_{\cluster_{l}}$ are zero-mean Gaussian 
	with variance upper bounded by $2 \sigma^{2} \| \mathbf{W} \|_{\infty} / \rho^{2}(\cluster_{l})$ (see \eqref{equ_bound_cols_pseudo_D}). 
	Therefore, \eqref{equ_concentration_lin_reg_sum_opt} and a union bound yields
	\begin{align}
	\label{equ_final_prob_bound_2}
	\prob \{\mbox{``\eqref{equ_condition_max_infty_norm_transformed_noise} invalid''}\}\!\leq\! 2M\exp \bigg(\!-\! \frac{M  \rho^2_{\partition} \lambda} {16 \sigma^2  \| \mathbf{D} \|^{2}_{\infty}}\bigg).
	\end{align}
	A union bound yields the upper bound \eqref{lower_boung_prob_main_results} by summing the bounds \eqref{equ_final_prob_bound_1} and \eqref{equ_final_prob_bound_2} .

\begin{lemma}
	\label{lem_bound_noise_term}
	Consider an empirical graph $\graph=(\nodes,\edges,\mW)$. For any two graph signals $\vu,\vv \in \graphsigs$,  
	\begin{align} 
	\label{equ_lem_bound_noise_term_lemma}
	\sum_{i \in \nodes} u_{i} v_{i} & \leq \nonumber \\
	& \hspace*{-13mm} (1/|\nodes|) \sum_{i \in \nodes} v_{i}  \sum_{j \in \nodes} u_{j} + \big\| \big( \mD^{\dagger} \big)^{T} \vv\big\|_{\infty} \| \vu \|_{\rm TV}. 
	\end{align}
	Here, $\mD \in \mathbb{R}^{\edges \times \nodes}$ denotes the incidence matrix of the graph $\graph$ under an arbitrary orientation of its edges $\edges$. 
\end{lemma} 
\begin{proof} 
	Any graph signal $\vu$ can be decomposed as 
	\begin{equation} 
	\label{equ_decomp_proj_orth_proj}
	\vu= \mathbf{P} \vu+ (\mathbf{I}- \mathbf{P}) \vu, 
	\end{equation} 
	with $\mathbf{P}$ denoting the orthogonal projection matrix on the nullspace of the graph Laplacian matrix 
	$\mL$ (see \eqref{equ_def_graph_Laplacian_matrix}). 
	
	For a connected graph, the nullspace $\mathcal{K}(\mL)$ is the one-dimensional subspace of 
	constant graph signals (see \cite{Luxburg2007}). Therefore, in this case, the projection $\mathbf{P}$ is given by 
	\begin{equation}
	\label{equ_def_projec_connected_gaph}
	\mathbf{P} = (1/(\mathbf{1}^{T} \mathbf{1})) \mathbf{1} \mathbf{1}^{T} = (1/|\nodes|)  \mathbf{1} \mathbf{1}^{T}
	\end{equation} 
	with the constant graph signal $\mathbf{1}$ assigning all nodes the same signal value $1$. 
	Therefore, 
	\begin{equation} 
	\label{equ_application_projection_nullspace}
	\mathbf{P} \vx  \stackrel{\eqref{equ_def_projec_connected_gaph}}{=} (1/|\nodes|) \mathbf{1} (\mathbf{1}^{T}\mathbf{u})  = (1/|\nodes|) \sum_{i \in \nodes} u_{i} \mathbf{1}.
	\end{equation}
	
	The projection on the orthogonal complement of the nullspace $\mathcal{K}(\mL) \subseteq \mathbb{R}^{\nodes}$ is given by $\mathbf{I} - \mathbf{P}$. 
	We can represent this projection conveniently using the incidence matrix $\mD$ \eqref{equ_def_incidence_matrix_elementwise} (see \cite{pmlr-v49-huetter16})
	\begin{equation}
	\label{equ_applicatio_ortho_projection}
	\mathbf{I}-\mathbf{P} = \mD^{\dagger} \mD.
	\end{equation}
	Combining \eqref{equ_application_projection_nullspace} and \eqref{equ_applicatio_ortho_projection} with \eqref{equ_decomp_proj_orth_proj}, 
	\begin{align}
	\label{equ_proof_basic_decomp_error}
	\hspace*{-3mm}\sum_{i \in \nodes} u_{i} v_{i} \!=\! 
	(1/|\nodes|) \sum_{i \in \nodes} u_{i} \sum_{j \in \nodes} v_{j}\!+\!\vv^{T} \mD^{\dagger} \mD \vu.
	\end{align}
	Combining \eqref{equ_proof_basic_decomp_error} with the inequality $\va^{T} \vb \leq \| \va \|_{\infty} \| \vb \|_{1}$,  
	\begin{align}
	\label{equ_proof_basic_decomp_error_12}
	\sum_{i \in \nodes} u_{i} v_{i} & \leq \nonumber \\
	& \hspace*{-14mm}  (1/|\nodes|) \sum_{i \in \nodes} u_{i} \sum_{j \in \nodes} v_{j}\!+\!\big\| \big( \mD^{\dagger} \big)^{T} \vv \big\|_{\infty} \| \mD \vu \|_{1}.
	\end{align} 
	The result \eqref{equ_lem_bound_noise_term_lemma} follows from \eqref{equ_proof_basic_decomp_error_12} by using the identity \eqref{equ_ident_TV_1_norm_D}. 
\end{proof} 

Applying Lemma \ref{lem_bound_noise_term} to the subgraphs $\graph_{\cluster}$ induced by a 
partition $\partition = \{\cluster_{1},\ldots,\cluster_{|\partition|} \}$, we obtain 
the following result. 
\begin{corollary}
	\label{cor_bound_noise_term_spectralgap}
	Consider an empirical graph $\graph=(\nodes,\edges,\mW)$ whose nodes are partitioned into disjoint clusters $\partition=\{\cluster_{1},\ldots,\cluster_{|\partition|}\}$. 
	We overload notation and denote by $\cluster_{l}$ also the subgraph induced by the nodes in $\cluster_{l}$ and assume that these 
	subgraphs are connected. Then, for any two graph signals $\vu, \vv \in \graphsigs$,
	\begin{align} 
	\label{equ_lem_bound_noise_term}
	\sum_{i \in \samplingset} v_{i} u_{i} & \!\leq\!  \max_{l=1,\ldots,|\partition|}| (1/|\cluster_{l}|) \sum_{i \in \cluster_{l} } v_{i}| \hspace*{-1mm}\sum_{j \in \samplingset} |u_{j}|  \nonumber \\ 
	& \hspace*{-10mm}+ \max_{l=1,\ldots,|\partition|} \big\| \big( \mD_{\cluster_{l}}^{\dagger} \big)^{T} \vv_{\cluster_{j}} \big\|_{\infty} \| \vu \|_{\rm TV}. 
	\end{align}
	Here, $\mD_{\cluster_{l}} \in \mathbb{R}^{\edges \times \nodes}$ denotes the incidence matrix of the 
	subgraph $\cluster_{l}$ under an arbitrary orientation of its edges. 
\end{corollary} 

\section{CONCLUSION}

Using a simple non-parametric regression model for network-structured datasets, 
we have derived an upper bound on the probability of the nLasso error to exceed a 
given threshold. This bound applies if the training set satisfies the NCC with respect 
to a partitioning of the empirical graph into clusters of data points with similar labels. 
The NCC is related to the existence of a sufficiently large flow between nodes of the training set 
and the boundaries between clusters in the dataset. Our analysis reveals how the 
accuracy of nLasso depends on the empirical graph structure. We have identified two key 
quantities which determine the required size of the training set. These quantities are the condition 
number associated with the NCC and the spectral gap of the cluster structure. A promising 
avenue for future work is the extension of our analysis of nLasso to more general probabilistic 
models for the data (labels). In particular we plan to study exponential families for the label distribution, 
which covers classification as well as multi-label problems. 

\bibliography{SLPBib}

\end{document}